\documentclass[letterpaper,11pt]{article}

\usepackage{graphicx}
\usepackage{subfigure}
\usepackage{fullpage}
\usepackage{amsmath}
\usepackage{amsbsy}
\usepackage{amssymb}
\usepackage{amsthm}
\usepackage{float}
\usepackage[margin = 1 in]{geometry}

\newcommand{\Mix}{{\mathcal{F}}}       
\newcommand{\wvar}{{\alpha}}       
\newcommand{\mx}{{\mu}}       
\newcommand{\Sx}{{\Sigma}}    
\newcommand{\ol}{{\phi}}         
\newcommand{\rol}{{\sqrt{\ol}}}    
\newcommand{\D}{{D}}          
\newcommand{\wt}{{\mathrm{w}}}    
\newcommand{\wmin}{{\wt}}    

\newcommand{\proj}{{\mathop{\rm proj}\nolimits}}

\newcommand{\Poly}{{\mathop{\rm poly}\nolimits}}

\newcommand{\Span}{{\mathop{\rm span}\nolimits}}

\newcommand{\Diag}{{\mathop{\rm diag}\nolimits}}
\newcommand{\Det}{{\mathop{\rm det}\nolimits}}

\newcommand{\bigexp}[1]{\exp\left({#1}\right)}
\newcommand{\oa}[1]{\bigexp{-\frac{\|{#1}\|^2}{2\wvar}}}
\newcommand{\ignore}[1]{}
\def\R{\mathrm R}
\def\eps{\epsilon}
\def\E{{\sf E}}
\def\Pr{{\sf P}}

\newtheorem{theorem}{Theorem}
\newtheorem{definition}{Definition}
\newtheorem{lemma}{Lemma}
\newtheorem{fact}[lemma]{Fact}

\newtheorem{claim}[lemma]{Claim}
\newtheorem{proposition}[lemma]{Proposition}

\floatstyle{ruled}
\newfloat{algorithm}{htbp}{loa}
\floatname{algorithm}{Algorithm}

\begin{document}

\title{Isotropic PCA and Affine-Invariant Clustering}
\author{S. Charles Brubaker\thanks{College of Computing, Georgia Tech.
Email: \tt {\{brubaker,vempala\}@cc.gatech.edu}}  \\
\and
Santosh S. Vempala\footnotemark[1]}

\date{}

\maketitle

\begin{abstract}
We present an extension of Principal Component Analysis (PCA) and a
new algorithm for clustering points in $\R^n$ based on it. The key
property of the algorithm is that it is affine-invariant.  When the
input is a sample from a mixture of two arbitrary Gaussians, the
algorithm correctly classifies the sample assuming only that the two
components are separable by a hyperplane, i.e., there exists a
halfspace that contains most of one Gaussian and almost none of the
other in probability mass. This is nearly the best possible, improving
known results substantially \cite{Arora2005, Kannan2005,
Achlioptas2005}. For $k>2$ components, the algorithm requires only
that there be some $(k-1)$-dimensional subspace in which the {\em
overlap} in every direction is small.  Here we define overlap to be
the ratio of the following two quantities: 1) the average squared
distance between a point and the mean of its component, and 2) the
average squared distance between a point and the mean of the
mixture. The main result may also be stated in the language of linear
discriminant analysis: if the standard Fisher discriminant
\cite{Duda2001} is small enough, labels are not needed to estimate the
optimal subspace for projection.  Our main tools are isotropic
transformation, spectral projection and a simple reweighting
technique.  We call this combination {\em isotropic PCA}.
\end{abstract}

\thispagestyle{empty} \setcounter{page}{0} \clearpage

\section{Introduction}
We present an extension to Principal Component Analysis (PCA), which
is able to go beyond standard PCA in identifying ``important''
directions. When the covariance matrix of the input (distribution or
point set in $\R^n$) is a multiple of the identity, then PCA reveals
no information; the second moment along any direction is the
same. Such inputs are called isotropic.  Our extension, which we call
{\em isotropic PCA}, can reveal interesting information in such
settings. We use this technique to give an affine-invariant clustering
algorithm for points in $\R^n$. When applied to the problem of
unraveling mixtures of arbitrary Gaussians from unlabeled samples, the
algorithm yields substantial improvements of known results.

To illustrate the technique, consider the uniform distribution on the
set $X = \{(x,y) \in \mathbb{R}^2 : x \in \{-1,1\}, y \in
[-\sqrt{3},\sqrt{3}]\}$, which is isotropic.  Suppose this
distribution is rotated in an unknown way and that we would like to
recover the original $x$ and $y$ axes.  For each point in a sample, we
may project it to the unit circle and compute the covariance matrix of
the resulting point set.  The $x$ direction will correspond to the
greater eigenvector, the $y$ direction to the other.  See Figure
\ref{fig:2d-eg} for an illustration.  Instead of projection onto the
unit circle, this process may also be thought of as importance
weighting, a technique which allows one to simulate one distribution
with another.  In this case, we are simulating a distribution over the
set $X$, where the density function is proportional to $(1 +
y^2)^{-1}$, so that points near $(1,0)$ or $(-1,0)$ are more probable.

\begin{figure}[h]
\center
\includegraphics[height=2in]{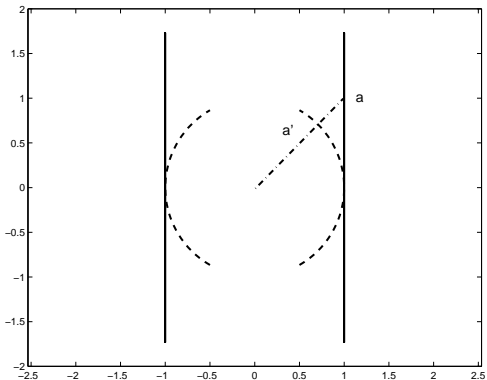}\label{fig:2d-eg}
\caption{Mapping points to the unit circle and then finding the
direction of maximum variance reveals the orientation of this
isotropic distribution.}
\end{figure}

In this paper, we describe how to apply this method to mixtures of
arbitrary Gaussians in $\mathbb{R}^n$ in order to find a set of
directions along which the Gaussians are well-separated.  These
directions span the Fisher subspace of the mixture, a classical
concept in Pattern Recognition.  Once these directions are identified,
points can be classified according to which component of the
distribution generated them, and hence all parameters of the mixture
can be learned.

What separates this paper from previous work on learning mixtures is
that our algorithm is affine-invariant. Indeed, for every mixture
distribution that can be learned using a previously known algorithm,
there is a linear transformation of bounded condition number that
causes the algorithm to fail.  For $k=2$ components our algorithm has
nearly the best possible guarantees (and subsumes all previous
results) for clustering Gaussian mixtures.  For $k > 2$, it requires
that there be a $(k-1)$-dimensional subspace where the \emph{overlap}
of the components is small in every direction (See section
\ref{sec:results}).  This condition can be stated in terms of the
Fisher discriminant, a quantity commonly used in the field of Pattern
Recognition with labeled data.  Because our algorithm is affine
invariant, it makes it possible to unravel a much larger set of
Gaussian mixtures than had been possible previously.

The first step of our algorithm is to place the mixture in isotropic
position (see Section \ref{sec:results}) via an affine
transformation. This has the effect of making the $(k-1)$-dimensional
Fisher subspace, i.e., the one that minimizes the Fisher discriminant,
the same as the subspace spanned by the means of the components (they
only coincide in general in isotropic position), for {\em any}
mixture. The rest of the algorithm identifies directions close to this
subspace and uses them to cluster, without access to
labels. Intuitively this is hard since after isotropy, standard PCA
reveals no additional information. Before presenting the ideas and
guarantees in more detail, we describe relevant related work.

\subsection{Previous Work}\label{sec:pw}
A mixture model is a convex combination of distributions of known
type. In the most commonly studied version, a distribution $F$ in
$\R^n$ is composed of $k$ unknown Gaussians.  That is,
\[
F = \wt_1 N(\mx_1,\Sx_1) + \ldots + \wt_k N(\mx_k,\Sx_k),
\]
where the mixing weights $\wt_i$, means $\mx_i$, and covariance
matrices $\Sx_i$ are all unknown.  Typically, $k \ll n$, so that a
concise model explains a high dimensional phenomenon.  A random sample
is generated from $F$ by first choosing a component with probability
equal to its mixing weight and then picking a random point from that
component distribution. In this paper, we study the classical problem
of unraveling a sample from a mixture, i.e., labeling each point in
the sample according to its component of origin.

Heuristics for classifying samples include ``expectation
maximization'' \cite{Dempster1977} and ``k-means clustering''
\cite{MacQueen1967}. These methods can take a long time and can get
stuck with suboptimal classifications.  Over the past decade, there
has been much progress on finding polynomial-time algorithms with
rigorous guarantees for classifying mixtures, especially mixtures of
Gaussians \cite{Dasgupta1999, Dasgupta2000, Arora2005, Vempala2002,
Kannan2005, Achlioptas2005}.  Starting with Dasgupta's paper
\cite{Dasgupta1999}, one line of work uses the concentration of
pairwise distances and assumes that the components' means are so far
apart that distances between points from the same component are likely
to be smaller than distances from points in different
components. Arora and Kannan \cite{Arora2005} establish nearly optimal
results for such distance-based algorithms. Unfortunately their results
inherently require separation that grows with the dimension of the
ambient space and the largest variance of each component Gaussian.

To see why this is unnatural, consider $k$ well-separated Gaussians in
$\mathbb{R}^k$ with means $e_1,\ldots, e_k$, i.e.  each mean is 1 unit
away from the origin along a unique coordinate axis.  Adding extra
dimensions with arbitrary variance does not affect the separability of
these Gaussians, but these algorithms are no longer guaranteed to
work.  For example, suppose that each Gaussian has a maximum variance
of $\epsilon \ll 1$.  Then, adding $O^*(k\epsilon^{-2})$ extra
dimensions with variance $\epsilon$ will violate the necessary
separation conditions.

To improve on this, a subsequent line of work uses spectral projection
(PCA). Vempala and Wang \cite{Vempala2002} showed that for a mixture
of {\em spherical} Gaussians, the subspace spanned by the top $k$
principal components of the mixture contains the means of the
components. Thus, projecting to this subspace has the effect of
shrinking the components while maintaining the separation between
their means. This leads to a nearly optimal separation requirement of
\[
\|\mu_i - \mu_j\| \ge \tilde{\Omega}(k^{1/4})
\max \{\sigma_{i}, \sigma_{j}\}
\]
where $\mu_i$ is the mean of component $i$ and $\sigma_{i}^2$ is
the variance of component $i$ along any direction.
Note that there is no dependence on the dimension of the
distribution. Kannan et al. \cite{Kannan2005} applied the spectral
approach to arbitrary mixtures of Gaussians (and more generally,
logconcave distributions) and obtained a separation that grows with a
polynomial in $k$ and the largest variance of each
component:
\[
\|\mu_i - \mu_j\| \ge \mbox{ poly}(k)
\max \{\sigma_{i,\max}, \sigma_{j,\max}\}
\]
where $\sigma_{i,\max}^2$ is the maximum variance of the $i$th
component in any direction. The polynomial in $k$ was improved in
\cite{Achlioptas2005} along with matching lower bounds for this
approach, suggesting this to be the limit of spectral methods. Going
beyond this ``spectral threshold" for arbitrary Gaussians has been a
major open problem.

The representative hard case is the special case of two parallel
``pancakes", i.e., two Gaussians that are spherical in $n-1$
directions and narrow in the last direction, so that a hyperplane
orthogonal to the last direction separates the two.  The spectral
approach requires a separation that grows with their largest standard
deviation which is unrelated to the distance between the pancakes
(their means).  Other examples can be generated by starting with
Gaussians in $k$ dimensions that are separable and then adding other
dimensions, one of which has large variance.  Because there is a
subspace where the Gaussians are separable, the separation requirement
should depend only on the dimension of this subspace and the
components' variances in it.

\begin{figure}
\begin{center}
\subfigure[Distance Concentration Separability]{
\includegraphics[width=2in]{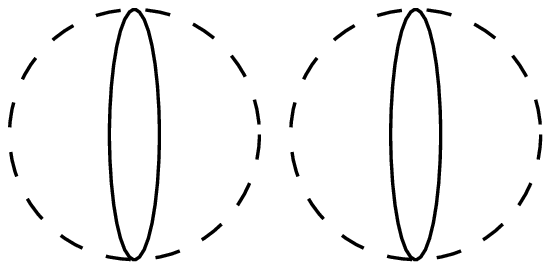}\label{fig:dc-sep}}
\subfigure[Hyperplane Separability]{
\includegraphics[width=2in]{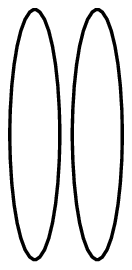}\label{fig:h-sep}}
\subfigure[Intermean Hyperplane and Fisher Hyperplane.]{
\includegraphics[width=2in]{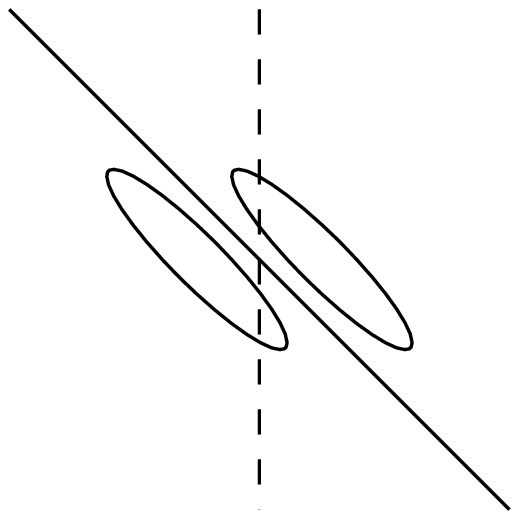}\label{fig:tilt}}
\caption{Previous work requires distance concentration separability
which depends on the maximum directional variance (a).  Our results
require only hyperplane separability, which depends only on the
variance in the separating direction(b).  For non-isotropic mixtures
the best separating direction may not be between the means of the
components(c).}\label{fig:pancakes}
\end{center}
\end{figure}

A related line of work considers learning symmetric product
distributions, where the coordinates are independent.  Feldman et al
\cite{Feldman2006} have shown that mixtures of axis-aligned Gaussians
can be approximated without any separation assumption at all in time
exponential in $k$.  A.  Dasgupta et al \cite{Dasgupta2005} consider
heavy-tailed distributions as opposed to Gaussians or log-concave ones
and give conditions under which they can be clustered using an
algorithm that is exponential in the number of samples.  Chaudhuri and
Rao \cite{Chaudhuri2008} have recently given a polynomial time
algorithm for clustering such heavy tailed product distributions.

\subsection{Results}\label{sec:results}
We assume we are given a lower bound $\wmin$ on the minimum mixing
weight and $k$, the number of components.  With high
probability, our algorithm \textsc{Unravel} returns a partition of
space by hyperplanes so that each part (a polyhedron) encloses almost
all of the probability mass of a single component and almost none of
the other components.  The error of such a set of polyhedra is the
total probability mass that falls outside the correct polyhedron.

We first state our result for two Gaussians in a way that makes clear the
relationship to previous work that relies on separation.

\begin{theorem}\label{thrm:k=2-sep}
Let $\wt_1,\mx_1,\Sx_1$ and $\wt_2,\mx_2,\Sx_2$ define a mixture of
two Gaussians.  There is an absolute constant $C$ such that, if there
exists a direction $v$ such that
\[
|\proj_v (\mu_1 - \mu_2)| \geq C \left(\sqrt{v^T\Sx_1v} + \sqrt{v^T\Sx_2v}\right)
\wmin^{-2} \log ^{1/2}\left( \frac{1}{\wmin \delta} + \frac{1}{\eta}\right),
\]
then with probability $1-\delta$ algorithm \textsc{Unravel} returns
two complementary halfspaces that have error at most $\eta$ using
time and a number of samples that is polynomial in
$n,\wmin^{-1},\log (1/\delta)$.
\end{theorem}
So the separation required between the means is comparable to the
standard deviation in {\em some direction}. This separation condition
of Theorem \ref{thrm:k=2-sep} is affine-invariant and much weaker than
conditions of the form $\| \mu_1 - \mu_2\| \gtrsim
\max\{\sigma_{1,\max}, \sigma_{2,\max}\}$ used in previous work.  See
Figure \ref{fig:pancakes}.  The dotted line shows how previous work
effectively treats every component as spherical.  We also note that
the separating direction does not need to be the intermean direction
as illustrated in Figure \ref{fig:tilt}.  The dotted line illustrates
hyperplane induced by the intermean direction, which may be far from
the optimal separating hyperplane shown by the solid line.

It will be insightful to state this result in terms of the Fisher
discriminant, a standard notion from Pattern Recognition
\cite{Duda2001,Fukunaga1990} that is used with labeled data.  In
words, the Fisher discriminant along direction $p$ is
\[
J(p) = \frac{\mbox{ the intra-component variance in direction $p$}}
{\mbox{the total variance in direction $p$}}
\]
Mathematically, this is expressed as
\[
J(p) = \frac{E\left[\| \proj_p(x - \mx_{\ell(x)})\|^2 \right]}
{E\left[\| \proj_p(x) \|^2 \right]} =
\frac{p^T( \wt_1 \Sx_1 + \wt_2 \Sx_2)p}
{p^T(\wt_1 (\Sx_1 + \mx_1\mx_1^T) + \wt_2 (\Sx_2 + \mx_2\mx_2^T))p}
\]
for $x$ distributed according to a mixture distribution with means
$\mu_i$ and covariance matrices $\Sx_i$.  We use $\ell(x)$ to indicate
the component from which $x$ was drawn.

\begin{theorem}\label{thrm:k=2-fisher}
There is an absolute constant $C$ for which the following holds.
Suppose that $\Mix$ is a mixture of two Gaussians such that there
exists a direction $p$ for which
\[
J(p) \leq C \wmin^3 \log^{-1} \left(\frac{1}{\delta\wmin} +
\frac{1}{\eta}\right).
\]
With probability $1-\delta$, algorithm \textsc{Unravel} returns a
halfspace with error at most $\eta$  using time and sample complexity
polynomial in $n,\wmin^{-1},\log(1/\delta)$.
\end{theorem}

There are several ways of generalizing the Fisher discriminant for
$k=2$ components to greater $k$ \cite{Fukunaga1990}.  These
generalizations are most easily understood when the distribution is
isotropic.  An isotropic distribution has the identity matrix as its
covariance and the origin as its mean. An isotropic mixture therefore
has
\[
\sum_{i=1}^k \wt_i \mx_i = 0\;\; \mbox{and}\;\; \sum_{i=1}^k \wt_i
(\Sx_i + \mx_i\mx_i^T) = I.
\]
It is well known that any distribution with bounded covariance matrix
(and therefore any mixture) can be made isotropic by an affine
transformation.  As we will see shortly, for $k=2$, for an isotropic
mixture, the line joining the means is the direction that minimizes
the Fisher discriminant.

Under isotropy, the denominator of the Fisher discriminant is always
$1$. Thus, the discriminant is just the expected squared distance
between the projection of a point and the projection of its mean,
where projection is onto some direction $p$.  The generalization to $k
> 2$ is natural, as we may simply replace projection onto direction
$p$ with projection onto a $(k-1)$-dimensional subspace $S$.  For
convenience, let
\[
\Sx= \sum_{i=1}^k \wt_i\Sx_i.
\]
Let the vector $p_1,\ldots,p_{k-1}$ be an orthonormal basis of $S$ and
let $\ell(x)$ be the component from which $x$ was drawn. We then have
under isotropy
\[
J(S) = E[ \|\proj_S(x - \mu_{\ell(x)})\|^2] =
\sum_{j=1}^{k-1} p_j^T \Sx p_j
\]
for $x$ distributed according to a mixture distribution with means
$\mu_i$ and covariance matrices $\Sx_i$.  As $\Sx$ is symmetric
positive definite, it follows that the smallest $k-1$ eigenvectors of
the matrix are optimal choices of $p_j$ and $S$ is the span of these
eigenvectors.

This motivates our definition of the Fisher subspace for \emph{any}
mixture with bounded second moments (not necessarily Gaussians).

\begin{definition}\label{def:fisher-subspace}
Let $\{\wt_i,\mx_i,\Sx_i\}$ be the weights, means, and covariance
matrices for an isotropic
\footnote{For non-isotropic mixtures, the Fisher discriminant
generalizes to $\sum_{j=1}^{k-1} p_j^T \left(\sum_{i=1}^{k} \wt_i
(\Sx_i + \mx_i\mx_i^T) \right)^{-1} \Sx p_j$ and the overlap to $p^T
\left(\sum_{i=1}^{k} \wt_i (\Sx_i + \mx_i\mx_i^T) \right)^{-1} \Sx p$}
mixture distribution with mean at the origin and where
$\dim(\Span\{\mx_1,\ldots,\mx_k\}) = k-1$.  Let $\ell(x)$ be the
component from which $x$ was drawn.  The {\em Fisher subspace} $F$ is
defined as the $(k-1)$-dimensional subspace that minimizes
\[
J(S) = E[ \|\proj_S(x - \mu_{\ell(x)})\|^2].
\]
over subspaces $S$ of dimension $k-1$.
\end{definition}

Note that $\dim(\Span\{\mx_1,\ldots,\mx_k\})$ is only $k-1$ because
isotropy implies $\sum_{i=1}^k \wt_i \mx_i = 0$.  The next lemma
provides a simple alternative characterization of the Fisher subspace
as the span of the means of the components (after transforming to
isotropic position).  The proof is given in Section \ref{sec:fisher}.
\begin{lemma}\label{lem:fisher-is-intermean}
Suppose $\{\wt_i,\mx_i,\Sx_i\}_{i=1}^k$ defines an isotropic mixture
in $\mathbb{R}^n$.  Let $\lambda_1 \geq \ldots \geq \lambda_n$ be the
eigenvalues of the matrix $\Sx = \sum_{i=1}^k \wt_i \Sx_i$ and let
$v_1,\ldots,v_n$ be the corresponding eigenvectors.  If the dimension
of the span of the means of the components is $k-1$, then the Fisher
subspace
\[
F = \Span\{v_{n-k+1},\ldots,v_n\} =
\Span\{\mx_1,\ldots,\mx_k\}.
\]
\end{lemma}
Our algorithm attempts to find the Fisher subspace (or one close to it) and succeeds in doing so, provided the discriminant is small enough.

The next definition will be useful in stating our main theorem precisely.
\begin{definition}\label{def:overlap}
The {\em overlap} of a mixture given as in Definition \ref{def:fisher-subspace} is
\begin{equation}\label{eqn:overlap}
\ol = \min_{S:\dim(S)=k-1} \max_{p \in S} p^T \Sx p.
\end{equation}
\end{definition}
It is a direct consequence of the Courant-Fisher min-max theorem that
$\ol$ is the $(k-1)$th smallest eigenvalue of the matrix $\Sx$ and the
subspace achieving $\ol$ is the Fisher subspace, i.e.,
\[
\ol = \left\| E[ \proj_F(x - \mx_{\ell(x)})
\proj_F(x - \mu_{\ell(x)})^T] \right\|_2.
\]
We can now state our main theorem
for $k > 2$.
\begin{theorem}\label{thrm:main}
There is an absolute constant $C$ for which the following holds.
Suppose that $\Mix$ is a mixture of $k$ Gaussian components where
the overlap satisfies
\[
\ol \leq C \wmin^3 k^{-3} \log^{-1} \left(\frac{nk}{\delta\wmin} +
\frac{1}{\eta}\right)
\]
With probability $1-\delta$, algorithm \textsc{Unravel} returns a set
of $k$ polyhedra that have error at most $\eta$ using time and a
number of samples that is polynomial in $n,\wmin^{-1},\log
(1/\delta)$.
\end{theorem}
In words, the algorithm successfully unravels arbitrary Gaussians
provided there exists a $(k-1)$-dimensional subspace in which along
every direction, the expected squared distance of a point to its
component mean is smaller than the expected squared distance to the
overall mean by roughly a $\Poly(k, 1/\wmin)$ factor. There is no
dependence on the largest variances of the individual components, and
the dependence on the ambient dimension is logarithmic.  This means
that the addition of extra dimensions (even where the distribution has
large variance) as discussed in Section \ref{sec:pw} has little impact
on the success of our algorithm.

\section{Algorithm}\label{sec:algorithm}
The algorithm has three major components: an initial affine
transformation, a reweighting step, and identification of a direction
close to the Fisher subspace and a hyperplane orthogonal to this
direction which leaves each component's probability mass almost
entirely in one of the halfspaces induced by the hyperplane.  The key
insight is that the reweighting technique will either cause the mean
of the mixture to shift in the intermean subspace, or cause the top
$k-1$ principal components of the second moment matrix to approximate
the intermean subspace.  In either case, we obtain a direction along
which we can partition the components.

We first find an affine transformation $W$ which when applied to
$\Mix$ results in an isotropic distribution.  That is, we move the
mean to the origin and apply a linear transformation to make the
covariance matrix the identity. We apply this transformation to a new
set of $m_1$ points $\{x_i\}$ from $\Mix$ and then reweight according
to a spherically symmetric Gaussian $\exp(-\|x\|^2/(2\wvar))$ for
$\wvar = \Theta(n/\wmin)$.  We then compute the mean $\hat{u}$ and second
moment matrix $\hat{M}$ of the resulting set.
\footnote{This practice of transforming the points and then looking at
the second moment matrix can be viewed as a form of kernel PCA;
however the connection between our algorithm and kernel PCA is
superficial.  Our transformation does not result in any standard
kernel.  Moreover, it is dimension-preserving (it is just a
reweighting), and hence the ``kernel trick'' has no computational
advantage.}

After the reweighting, the algorithm chooses either the new mean or
the direction of maximum second moment and projects the data onto this
direction $h$.  By bisecting the largest gap between points, we obtain
a threshold $t$, which along with $h$ defines a hyperplane that
separates the components.  Using the notation $H_{h,t} = \{x \in
\mathbb{R}^n : h^T x \geq t\}$, to indicate a halfspace, we then
recurse on each half of the mixture.  Thus, every node in the
recursion tree represents an intersection of half-spaces.  To make our
analysis easier, we assume that we use different samples for each step
of the algorithm.  The reader might find it useful to read Section
\ref{sec:pancakes}, which gives an intuitive explaination for how the
algorithm works on parallel pancakes, before reviewing the details of
the algorithm.

\begin{algorithm*}
\caption{Unravel} \label{alg:unravel}
Input: Integer $k$, scalar $\wmin$.
Initialization: $P = \mathbb{R}^n$.
\begin{enumerate}
\item (Isotropy) Use samples lying in $P$ to compute an
affine transformation $W$ that makes the distribution nearly isotropic
(mean zero, identity covariance matrix).

\item (Reweighting) Use $m_1$ samples in $P$ and for each
compute a weight $e^{-\|x\|^2/(\alpha)}$ (where $\alpha > n/\wmin$).

\item (Separating Direction) Find the mean of the reweighted data
$\hat{\mu}$.  If $\|\hat{\mu}\| > \sqrt{\wmin}/(32 \alpha)$, let $h =
\hat{\mu}$.  Otherwise, find the covariance matrix $\hat{M}$ of the
reweighted points and let $h$ be its top principal component.

\item (Recursion) Project $m_2$ sample
points to $h$ and find the largest gap between points in the interval
$[-1/2,1/2]$.  If this gap is less than $1/4(k-1)$, then return
$P$.  Otherwise, set $t$ to be the midpoint of the largest gap,
recurse on $ P \cap H_{h,t}$ and $P \cap H_{-h,-t}$, and return the
union of the polyhedra produces by these recursive calls.
\end{enumerate}
\end{algorithm*}







\subsection{Parallel Pancakes}\label{sec:pancakes}
The following special case, which represents the open problem in
previous work, will illuminate the intuition behind the new algorithm.
Suppose $\Mix$ is a mixture of two spherical Gaussians that are
well-separated, i.e. the intermean distance is large compared to the
standard deviation along any direction.  We consider two cases, one
where the mixing weights are equal and another where they are
imbalanced.

After isotropy is enforced, each component will become thin in the
intermean direction, giving the density the appearance of two
parallel pancakes.  When the mixing weights are equal, the means of
the components will be equally spaced at a distance of $1 -
\ol$ on opposite sides of the origin.  For imbalanced weights, the
origin will still lie on the intermean direction but will be much
closer to the heavier component, while the lighter component will be
much further away.  In both cases, this transformation makes the
variance of the mixture 1 in every direction, so the principal
components give us no insight into the inter-mean direction.

Consider next the effect of the reweighting on the mean of the
mixture.  For the case of equal mixing weights, symmetry assures that
the mean does not shift at all.  For imbalanced weights, however, the
heavier component, which lies closer to the origin will become heavier
still.  Thus, the reweighted mean shifts toward the mean of the
heavier component, allowing us to detect the intermean direction.

Finally, consider the effect of reweighting on the second moments of
the mixture with equal mixing weights.  Because points closer to the
origin are weighted more, the second moment in every direction is
reduced.  However, in the intermean direction, where part of the
moment is due to the displacement of the component means from the
origin, it shrinks less.  Thus, the direction of maximum second
moment is the intermean direction.

\subsection{Overview of Analysis}
To analyze the algorithm, in the general case, we will proceed as
follows.  Section \ref{sec:prelim} shows that under isotropy the
Fisher subspace coincides with the intermean subspace (Lemma
\ref{lem:fisher-is-intermean}), gives the necessary sampling
convergence and perturbation lemmas and relates overlap to a more
conventional notion of separation (Prop. \ref{prop:geo-2}). Section
\ref{sec:M-approx} gives approximations to the first and second
moments.  Section \ref{sec:find-dir} then combines these approximations
with the perturbation lemmas to show that the vector $h$ (either the
mean shift or the largest principal component) lies close to the
intermean subspace.  Finally, Section \ref{sec:recursion} shows the
correctness of the recursive aspects of the algorithm.

\section{Preliminaries}\label{sec:prelim}
\subsection{Matrix Properties}
For a matrix $Z$, we will denote the $i$th largest eigenvalue of $Z$
by $\lambda_i(Z)$ or just $\lambda_i$ if the matrix is clear from
context.  Unless specified otherwise, all norms are the 2-norm.  For
symmetric matrices, this is $\|Z\|_2 = \lambda_1(Z) = \max_{x \in
\mathbb{R}^n} \|Zx\|_2 /\|x\|_2$.

The following two facts from linear algebra will be useful in our
analysis.
\begin{fact}\label{fact:weilandt}
Let $\lambda_1 \geq \ldots \geq \lambda_n$ be the eigenvalues for an
$n$-by-$n$ symmetric positive definite matrix $Z$ and let $v_1,\ldots
v_n$ be the corresponding eigenvectors.  Then
\[
\lambda_n + \ldots +\lambda_{n-k+1} = \min_{S:\dim(S) = k}
\sum_{j=1}^k p_j^T Z p_j,
\]
where $\{p_j\}$ is any orthonormal basis for $S$.  If $\lambda_{n-k} >
\lambda_{n-k+1}$, then $\Span\{v_n,\ldots,v_{n-k+1}\}$ is the unique
minimizing subspace.
\end{fact}
Recall that a matrix $Z$ is positive semi-definite if $x^T Zx \ge
0$ for all non-zero $x$.

\begin{fact}\label{fact:off-diag-bound}
Suppose that the matrix
\[
Z = \left[
\begin{array}{cc}
A & B^T\\
B & D
\end{array}
\right]
\]
is symmetric positive semi-definite and that $A$ and $D$ are square
submatrices.  Then $\|B\| \leq \sqrt{\|A\|\|D\|}$.
\end{fact}
\begin{proof}
Let $y$ and $x$ be the top left and right singular vectors of $B$, so
that $y^T B x = \|B\|$.  Because $Z$ is positive semi-definite, we have
that for any real $\gamma$,
\[
0 \leq [\gamma x^T \; y^T] Z [\gamma x^T \; y^T]^T = \gamma^2 x^T A x + 2 \gamma y^T B x
+ y^T D y.
\]
This is a quadratic polynomial in $\gamma$ that can have only one real
root.  Therefore the discriminant must be non-positive:
\[
0 \geq 4(y^T B x)^2 - 4 (x^T A x ) (y^T D y).
\]
We conclude that
\[
\|B\| = y^T B x \leq \sqrt{(x^T A x)(y^T D y)} \leq \sqrt{\|A\|\|D\|}.
\]
\end{proof}

\subsection{The Fisher Criterion and Isotropy}\label{sec:fisher}

We begin with the proof of the lemma that for an isotropic mixture the Fisher subspace is the same as the intermean subspace.

\begin{proof}[Proof of Lemma \ref{lem:fisher-is-intermean}.]
By definition for an isotropic distribution, the Fisher subspace minimizes
\[
J(S) = E[ \|\proj_S(x - \mu_{\ell(x)})\|^2] =
\sum_{j=1}^{k-1} p_j^T \Sx p_j,
\]
where $\{p_j\}$ is an orthonormal basis for $S$.

By Fact \ref{fact:weilandt} one minimizing subspace is the span of the
smallest $k-1$ eigenvectors of the matrix $\Sx$, i.e. $v_{n-k+2},\ldots,v_n$.  Because the distribution is
isotropic,
\[
\Sx = I - \sum_{i=1}^k \wt_i \mx_i\mx_i^T,
\]
and these vectors become the largest eigenvectors of $\sum_{i=1}^k
\wt_i \mx_i\mx_i^T$.  Clearly, $\Span\{v_{n-k+2},\ldots,v_n\}
\subseteq \Span\{\mx_1,\ldots,\mx_k\}$, but both spans have dimension
$k-1$ making them equal.

Since $v_{n-k+1}$ must be orthogonal to the other eigenvectors, it
follows that $\lambda_{n-k+1} = 1 > \lambda_{n-k+2}$.  Therefore,
$\Span\{v_{n-k+2},\ldots,v_n\} \subseteq \Span\{\mx_1,\ldots,\mx_k\}$
is the unique minimizing subspace.
\end{proof}

It follows directly that
under the conditions of Lemma \ref{lem:fisher-is-intermean},
the overlap may be characterized as
\[
\ol = \lambda_{n-k+2}\left(\Sx\right)
= 1 - \lambda_{k-1} \left(\sum_{i=1}^k \wt_i \mx_i\mx_i^T \right).
\]

For clarity of the analysis, we will assume that Step 1 of the
algorithm produces a perfectly isotropic mixture.  Theorem
\ref{thrm:isotropy} gives a bound on the required number of samples to make
the distribution nearly isotropic, and as our analysis shows, our
algorithm is robust to small estimation errors.

We will also assume for convenience of notation that the the unit
vectors along the first $k-1$ coordinate axes $e_1, \ldots e_{k-1}$
span the intermean (i.e. Fisher) subspace.  That is, $F =
\Span\{e_1,\ldots,e_{k-1}\}$.  When considering this subspace it will be
convenient to be able to refer to projection of the mean vectors to
this subspace.  Thus, we define $\tilde{\mx_i} \in R^{k-1}$ to be the
first $k-1$ coordinates of $\mx_i$; the remaining coordinates are all
zero.  In other terms,
\[
\tilde{\mx_i} =
\begin{array}{cc}
\left[ I_{k-1} \right.& \left. 0 \right] \mx_i
\end{array}.
\]

In this coordinate system the covariance matrix of each component has
a particular structure, which will be useful for our analysis.
For the rest of this paper we fix the following notation:
an isotropic mixture is defined by $\{\wt_i,\mx_i,\Sx_i\}$. We assume that
$\Span\{e_1,\ldots,e_{k-1}\}$ is the intermean subspace and
$A_i$,$B_i$, and $D_i$ are defined such that
\begin{equation}\label{eqn:covar-structure}
\begin{array}{ccc}
\wt_i \Sx_i & = &
\left[
\begin{array}{cc}
A_i & B_i^T\\
B_i & \D_i
\end{array}
\right]
\end{array}
\end{equation}
where $A_i$ is a $(k-1) \times (k-1)$ submatrix and $D_i$ is a
$(n-k+1) \times (n-k+1)$ submatrix. \begin{lemma}[Covariance
Structure]\label{lemma:covar-structure} Using the above notation,
\[
\|A_i\| \le \ol \;\;,  \|D_i\| \le 1 \;\;, \|B_i\| \le \rol
\]
for all components $i$.
\end{lemma}

\begin{proof}[Proof of Lemma \ref{lemma:covar-structure}.]
Because $\Span\{e_1,\ldots,e_{k-1}\}$ is the Fisher subspace
\[
\phi = \max_{v \in \mathbb{R}^{k-1}} \frac{1}{\|v\|^2} \sum_{i=1}^k
v^T A_i v = \left\|\sum_{i=1}^k A_i
\right\|_2.
\]
Also $\sum_{i=1}^k D_i = I$, so $\|\sum_{i=1}^k D_i\| = 1$.  Each
matrix $\wt_i \Sx_i$ is positive definite, so the principal minors
$A_i$,$D_i$ must be positive definite as well. Therefore, $\|A_i\|
\leq \ol$, $\|D_i\| \leq 1$, and $\|B_i\| \leq \sqrt{\|A_i\|\|D_i\|} =
\rol$ using Fact \ref{fact:off-diag-bound}.
\end{proof}

For small $\ol$, the covariance between intermean and non-intermean
directions, i.e. $B_i$, is small.  For $k=2$, this means that all
densities will have a ``nearly parallel pancake'' shape.  In general,
it means that $k-1$ of the principal axes of the Gaussians will lie
close to the intermean subspace.

We conclude this section with a proposition connecting,
for $k=2$, the overlap to a standard notion of
separation between two distributions, so that Theorem
\ref{thrm:k=2-sep} becomes an immediate corollary of Theorem
\ref{thrm:k=2-fisher}.

\begin{proposition}\label{prop:geo-2}
If there exists a unit vector $p$ such that
\[
|p^T(\mx_1- \mx_2)| > t (\sqrt{p^T  \wt_1 \Sx_1p} + \sqrt{ p^T \wt_2\Sx_2 p}),
\]
then the overlap $\ol \leq J(p) \leq (1 + \wt_1 \wt_2 t^2)^{-1}$.
\end{proposition}

\begin{proof}[Proof of Proposition \ref{prop:geo-2}.]
Since the mean of the distribution is at the origin, we have $\wt_1 p^T
\mx_1 = -\wt_2 p^T \mx_2$.  Thus,
\begin{eqnarray*}
| p^T\mx_1- p^T\mx_2 |^2 & = &
(p^T\mx_1)^2 + (p^T\mx_2)^2 + 2|p^T\mx_1||p^T\mx_2|\\
& = & (\wt_1 p^T\mx_1)^2 \left(\frac{1}{\wt_1^2} + \frac{1}{\wt_2^2} + \frac{2}{\wt_1\wt_2}\right),
\end{eqnarray*}
using $\wt_1+\wt_2 = 1$.  We rewrite the last factor as
\[
\frac{1}{\wt_1^2} + \frac{1}{\wt_2^2} + \frac{2}{\wt_1\wt_2} = \frac{\wt_1^2 + \wt_2^2 + 2\wt_1\wt_2}{\wt_1^2\wt_2^2} = \frac{1}{\wt_1^2\wt_2^2} =
\frac{1}{\wt_1\wt_2}\left(\frac{1}{\wt_1} + \frac{1}{\wt_2}\right).
\]
Again, using the fact that $\wt_1 p^T \mx_1 = -\wt_2 p^T \mx_2$, we
have that
\begin{eqnarray*}
| p^T\mx_1- p^T\mx_2 |^2
& = & \frac{(\wt_1 p^T\mx_1)^2}{\wt_1\wt_2}
\left(\frac{1}{\wt_1} + \frac{1}{\wt_2}\right)\\
& = &\frac{\wt_1 (p^T\mx_1)^2 + \wt_2 (p^T\mx_2)^2}{\wt_1\wt_2}.
\end{eqnarray*}

Thus, by the separation condition
\[
\wt_1 (p^T\mx_1)^2 + \wt_2 (p^T\mx_2)^2 = \wt_1 \wt_2 | p^T\mx_1- p^T\mx_2 |^2
\geq \wt_1 \wt_2 t^2  ( p^T  \wt_1 \Sx_1p +  p^T \wt_2\Sx_2 p).
\]

To bound $J(p)$, we then argue
\begin{eqnarray*}
J(p) & = & \frac{ p^T \wt_1 \Sx_1 p +  p^T \wt_2 \Sx_2 p}
{\wt_1 ( p^T\Sx_1p + (p^T\mx_1)^2) + \wt_2 (p^T\Sx_2p + (p^T\mx_2)^2)}\\
& = & 1 - \frac{\wt_1 (p^T\mx_1)^2 + \wt_2 (p^T\mx_2)^2}
{\wt_1 ( p^T\Sx_1p + (p^T\mx_1)^2) + \wt_2 (p^T\Sx_2p + (p^T\mx_2)^2)}\\
& \leq & 1 - \frac{\wt_1 \wt_2 t^2(\wt_1 p^T\Sx_1p + \wt_2 p^T\Sx_2p)}
{\wt_1 ( p^T\Sx_1p + (p^T\mx_1)^2) + \wt_2 (p^T\Sx_2p + (p^T\mx_2)^2)}\\
&\leq & 1 - \wt_1\wt_2 t^2 J(p),
\end{eqnarray*}
and $J(p) \leq 1/(1 + \wt_1\wt_2 t^2)$.
\end{proof}

\subsection{Approximation of the Reweighted Moments}\label{sec:M-approx}
Our algorithm works by computing the first and second reweighted
moments of a point set from $\Mix$.  In this section, we examine
how the reweighting affects the second moments of a single component
and then give some approximations for the first and second moments of
the entire mixture.

\subsubsection{Single Component}
The first step is to characterize how the reweighting affects the moments
of a single component.  Specifically, we will show for any function
$f$ (and therefore $x$ and $xx^T$ in particular) that for $\wvar >
0$,
\[
E\left[ f(x) \oa{x}\right] = \sum_i \wt_i \rho_i E_i\left[ f(y_i) \right],
\]
Here, $E_i[\cdot]$ denotes expectation taken with respect to the
component $i$, the quantity $\rho_i = E_i\left[\oa{x}\right]$, and
$y_i$ is a Gaussian variable with parameters slightly perturbed from
the original $i$th component.

\begin{claim}\label{alpha-is-large}
If $\alpha = n/\wmin$, the quantity $\rho_i = E_i\left[\oa{x}\right]$
is at least $1/2$.
\end{claim}
\begin{proof}
Because the distribution is isotropic, for any component $i$, $\wt_i
E_i[\|x\|^2] \leq n$.  Therefore,
\[
\rho_i = E_i\left[\oa{x}\right]
\geq  E_i\left[1 - \frac{\|x\|^2}{2\wvar}\right]
\geq  1 - \frac{1}{2\wvar} \frac{n}{\wt_i}
\geq  \frac{1}{2}.
\]
\end{proof}

\begin{lemma}[Reweighted Moments of a Single Component]\label{lemma:moment-approx}
For any $\wvar > 0$, with respect to a single component $i$ of the
mixture
\[
E_i \left[x \oa{x} \right] = \rho_i(\mx_i - \frac{1}{\wvar} \Sx_i \mx_i + f)
\]
and
\[
E_i \left[xx^T \oa{x} \right] =
\rho(\Sx_i + \mx_i\mx_i^T - \frac{1}{\wvar}(\Sx_i\Sx_i + \mx_i\mx_i^T\Sx_i + \Sx_i\mx_i\mx_i^T) + F)
\]
where $\|f\|,\|F\| = O(\wvar^{-2})$.
\end{lemma}

We first establish the following claim.
\begin{claim}\label{claim:perturb}
Let $x$ be a random variable distributed according to the normal
distribution $N(\mx, \Sx)$ and let $\Sx = Q \Lambda Q^T$ be the singular value decomposition of $\Sx$ with $\lambda_1, \ldots, \lambda_n$ being the
diagonal elements of $\Lambda$.  Let $W = \Diag( \wvar/(\wvar +
\lambda_1), \ldots, \wvar/(\wvar + \lambda_n))$.  Finally, let $y$
be a random variable distributed according to $N(Q W Q^T \mx , Q
W \Lambda Q^T)$.  Then for any function $f(x)$,
\[
E\left[ f(x) \oa{x} \right]\\
= \Det(W)^{1/2}
\bigexp{-\frac{\mx^TQ W Q^T \mx}{2\wvar}}
E\left[f(y)\right].
\]
\end{claim}

\begin{proof}[Proof of Claim \ref{claim:perturb}]
We assume that $Q = I$ for the initial part of
the proof.  From the definition of a Gaussian distribution, we have
\[
E\left[f(x) \oa{x} \right] =  \Det(\Lambda)^{-1/2} (2\pi)^{-n/2}
\int_{\mathbb{R}^n} f(x) \exp\left(-\frac{x^Tx}{2\wvar}- \frac{(x - \mx)^T \Lambda^{-1}(x - \mx)}{2}
\right).
\]
Because $\Lambda$ is diagonal, we may write the exponents on the right
hand side as
\[
\sum_{i=1}^n x_i^2 \wvar^{-1} + (x_i - \mx_i)^2 \lambda_i^{-1} =
\sum_{i=1}^n x_i^2 (\lambda^{-1} + \wvar^{-1})- 2 x_i \mx_i \lambda_i^{-1}
+ \mx_i^2 \lambda_i^{-1}.
\]
Completing the square gives the expression
\[
\sum_{i=1}^n \left(x_i - \mu_i \frac{\wvar}{\wvar + \lambda_i}\right)^2
\left(\frac{\lambda_i \wvar}{\wvar + \lambda_i}\right)^{-1}
+ \mx_i^2 \lambda_i^{-1}
- \mx_i^2 \lambda_i^{-1} \frac{\wvar}{\wvar + \lambda_i}.
\]
The last two terms can be simplified to $\mx_i^2/(\wvar + \lambda_i)$.
In matrix form the exponent becomes
\[
\left(x - W\mx\right)^T
(W\Lambda)^{-1}
\left(x - W\mx\right)
+ \mx^T W \mx \wvar^{-1}.
\]
For general $Q$, this becomes
\[
\left(x - QWQ^{T} \mx\right)^T
 Q (W\Lambda)^{-1} Q^{T}
\left(x - Q W Q^T \mx\right)
+ \mx^T Q W Q^{T}\mx \wvar^{-1}.
\]
Now recalling the definition of the random variable $y$, we see
\begin{multline*}
E\left[f(x) \oa{x} \right] =
\Det(\Lambda)^{-1/2} (2\pi)^{-n/2}
\exp\left( -\frac{\mx^T Q W Q^T \mx}{2\wvar} \right) \\
\int_{\mathbb{R}^n} f(x) \exp\left(-\frac{1}{2}
\left(x - Q W Q^T\mx\right)^T
Q (W\Lambda)^{-1} Q^T
\left(x - Q W Q^T\mx\right)
\right)\\
= \Det(W)^{1/2}
\exp\left( -\frac{\mx^T Q W Q^T \mx}{2\wvar}\right)
E\left[f(y)\right].
\end{multline*}
\end{proof}

The proof of Lemma \ref{lemma:moment-approx} is now straightforward.

\begin{proof}[Proof of Lemma \ref{lemma:moment-approx}]
For simplicity of notation, we drop the subscript $i$ from $\rho_i$,
$\mx_i$, $\Sx_i$ with the understanding that all statements of
expectation apply to a single component.  Using the notation of Claim
\ref{claim:perturb}, we have
\[
\rho = E \left[\oa{x} \right] = \Det(W)^{1/2}
\bigexp{-\frac{\mx^T Q W Q^T \mx}{2\wvar}}.
\]
A diagonal entry of the matrix $W$ can expanded as
\[
\frac{\wvar}{\wvar + \lambda_i}
= 1 - \frac{\lambda_i}{\wvar + \lambda_i}
= 1 - \frac{\lambda_i}{\wvar} + \frac{\lambda_i^2}{\wvar(\wvar + \lambda_i)},
\]
so that
\[
W = I - \frac{1}{\wvar} \Lambda + \frac{1}{\wvar^2} W \Lambda^2.
\]
Thus,
\begin{eqnarray*}
E\left[ x \oa{x} \right] & = & \rho(Q W Q^T \mx)\\
 & = & \rho(Q I Q^T\mx - \frac{1}{\wvar} Q \Lambda Q^T\mx
+ \frac{1}{\wvar^2} Q W \Lambda^{2} Q^T\mx)\\
& = & \rho(\mx - \frac{1}{\wvar} \Sx \mx + f),
\end{eqnarray*}
where $\|f\| = O(\wvar^{-2})$.

We analyze the perturbed covariance in a similar fashion.
\begin{eqnarray*}
E\left[xx^T \oa{x}\right] & = &
\rho \left(Q (W \Lambda) Q^T + Q W Q^T \mx\mx^T Q W Q^T \right)\\
& = & \rho\left(Q \Lambda Q^T - \frac{1}{\wvar}Q\Lambda^2Q^T
+ \frac{1}{\wvar^2}Q W \Lambda^{3}Q^T \right.\\
& & \left.+ (\mx - \frac{1}{\wvar} \Sx \mx + f)
(\mx - \frac{1}{\wvar} \Sx \mx + f)^T \right)\\
& = & \rho\left(\Sx + \mx\mx^T
- \frac{1}{\wvar} (\Sx\Sx + \mx\mx^T\Sx + \Sx\mx\mx^T) + F\right),
\end{eqnarray*}
where $\|F\| = O(\wvar^{-2})$.
\end{proof}

\subsubsection{Mixture moments}
The second step is to approximate the first and second moments of the
entire mixture distribution.  Let $\rho$ be the vector where $\rho_i =
E_i\left[\oa{x}\right]$ and let $\bar{\rho}$ be the average of the
$\rho_i$.  We also define
\begin{eqnarray}
u & \equiv & E\left[x \oa{x}\right] =
\sum_{i=1}^k \wt_i \rho_i \mx_i
- \frac{1}{\wvar} \sum_{i=1}^k \wt_i \rho_i \Sx_i\mx_i + f \label{eqn:u}\\
M & \equiv & E\left[xx^T \oa{x}\right] =
\sum_{i=1}^k \wt_i \rho_i (\Sx_i + \mx_i\mx_i^T - \frac{1}{\wvar}(\Sx_i\Sx_i + \mx_i\mx_i^T\Sx_i + \Sx_i\mx_i\mx_i^T)) + F \label{eqn:M}
\end{eqnarray}
with $\|f\| = O(\wvar^{-2})$ and $\|F\| = O(\wvar^{-2})$.  We denote
the estimates of these quantities computed from samples by $\hat{u}$
and $\hat{M}$ respectively.

\begin{lemma}\label{lemma:M1-approx}
Let $v = \sum_{i=1}^k \rho_i \wt_i \mx_i.$ Then
\[
\|u - v \|^2 \leq \frac{4k^2 }{\wvar^2 \wmin} \ol .
\]
\end{lemma}
\begin{proof}[Proof of Lemma \ref{lemma:M1-approx}]
We argue from Eqn. \ref{eqn:covar-structure} and Eqn. \ref{eqn:u} that
\begin{eqnarray*}
\|u - v\| & = & \frac{1}{ \wvar} \left\| \sum_{i=1}^k \wt_i \rho_i
\Sx_i \mx_i \right\| + O(\wvar^{-2})\\
& \leq & \frac{1}{\wvar \sqrt{\wmin}}
\sum_{i=1}^k \rho_i \|(\wt_i\Sx_i) (\sqrt{\wt_i}\mx_i)\|+ O(\wvar^{-2})\\
& \leq &
\frac{1}{\wvar \sqrt{\wmin}} \sum_{i=1}^k \rho_i \|[A_i,B_i^T]^T\|\|
(\sqrt{\wt_i}\mx_i)\| + O(\wvar^{-2}).
\end{eqnarray*}
From isotropy, it follows that $\|\sqrt{\wt_i}\mx_i\| \leq 1$.  To bound
the other factor, we argue
\[
\|[A_i,B_i^T]^T\| \leq \sqrt{2} \max\{\|A_i\|, \|B_i\|\} \leq \sqrt{2\ol}.
\]
Therefore,
\[
\|u - v\|^2 \leq \frac{2k^2}{\wvar^2 \wmin} \ol + O(\alpha^{-3}) \leq
\frac{4k^2}{\wvar^2 \wmin} \ol,
\]
for sufficiently large $n$, as $\alpha \geq n/\wmin$.
\end{proof}

\begin{lemma}\label{lemma:M2-approx}
Let
\[
\Gamma =
\left[
\begin{array}{cc}
\sum_{i=1}^k \rho_i ( \wt_i \tilde{\mx_i}\tilde{\mx_i}^T + A_i) & 0\\
0 & \sum_{i=1}^k \rho_i D_i - \frac{\rho_i}{\wt_i\wvar} D_i^2\\
\end{array}
\right].
\]
If $\|\rho -1\bar{\rho}\|_\infty < 1/(2 \wvar)$, then
\[
\|M -\Gamma\|_2^2 \leq \frac{16^2 k^2}{\wmin^2\wvar^2} \ol.
\]
\end{lemma}

Before giving the proof, we summarize some of the necessary
calculation in the following claim.

\begin{claim}\label{claim:M-is-Gamma-plus-Delta}
The matrix of second moments
\[
M = E\left[xx^T \oa{x}\right] =
\left[
\begin{array}{cc}
\Gamma_{11} & 0 \\
0 & \Gamma_{22}
\end{array}\right]
+
\left[
\begin{array}{cc}
\Delta_{11} & \Delta_{21}^T \\
\Delta_{21} & \Delta_{22}
\end{array}\right] + F,
\]
where
\begin{eqnarray*}
\Gamma_{11} & = &
\sum_{i=1}^k \rho_i ( \wt_i \tilde{\mx_i}\tilde{\mx_i}^T + A_i)\\
\Gamma_{22} & = & \sum_{i=1}^k \rho_i D_i - \frac{\rho_i}{\wt_i\wvar} D_i^2\\
\\
\Delta_{11} & = & -\sum_{i=1}^k \frac{\rho_i}{\wt_i\wvar} B_i^T B_i
+ \frac{\rho_i}{\wt_i\wvar} \left(\wt_i\tilde{\mx_i}\tilde{\mx_i}^T A_i +
\wt_i A_i \tilde{\mx_i}\tilde{\mx_i}^T  + A_i^2 \right)\\
\Delta_{21} & = & \sum_{i=1}^k \rho_i B_i
- \frac{\rho_i}{\wt_i\wvar}
\left( B_i (\wt_i \tilde{\mx_i} \tilde{\mx_i}^T)
+ B_i A_i + D_i B_i\right)\\
\Delta_{22} & = & -\sum_{i=1}^k \frac{\rho_i}{\wt_i\wvar} B_i B_i^T,
\end{eqnarray*}
and $\|F\| = O(\alpha^{-2})$.
\end{claim}
\begin{proof}
The calculation is straightforward.
\end{proof}

\begin{proof}[Proof of Lemma \ref{lemma:M2-approx}]
We begin by bounding the 2-norm of each of the blocks.  Since
$\|\wt_i\tilde{\mx_i}\tilde{\mx_i}^T\| < 1$ and $\|A_i\| \leq \ol$ and
$\|B_i\| \leq \rol$, we can bound
\begin{eqnarray*}
\|\Delta_{11}\| & = &
\max_{\|y\| = 1} \sum_{i=1}^k
\frac{\rho_i}{\wt_i\wvar} y^TB_i^T B_iy^T
- \frac{\rho_i}{\wt_i\wvar} y^T \left(\wt_i\tilde{\mx_i}\tilde{\mx_i}^T A_i +
\wt_i A_i \tilde{\mx_i}\tilde{\mx_i}^T  + A_i^2 \right)y + O(\wvar^{-2})\\
& \leq & \sum_{i=1}^k
\frac{\rho_i}{\wt_i\wvar} \|B_i\|^2 +
\frac{\rho_i}{\wt_i\wvar} (2\|A\| + \|A\|^2) + O(\wvar^{-2})\\
& \leq & \frac{4k}{\wmin\wvar} \ol + O(\wvar^{-2}).
\end{eqnarray*}
By a similar argument, $\|\Delta_{22}\| \leq k\ol/(\wmin\wvar) +
O(\alpha^{-2})$.  For $\Delta_{21}$, we observe that $\sum_{i=1}^k B_i
= 0$.  Therefore,
\begin{eqnarray*}
\|\Delta_{21}\|& \leq &
\left\|\sum_{i=1}^k (\rho_i - \bar{\rho}) B_i \right\| +
\left\| \sum_{i=1}^k
\frac{\rho_i}{\wt_i\wvar}
\left(B_i (\wt_i\tilde{\mx}_i\tilde{\mx}_i^T) +
B_i A_i + D_i B_i\right) \right\| + O(\wvar^{-2}) \\
& \leq &
\sum_{i=1}^k |\rho_i - \bar{\rho}| \|B_i\|
+ \sum_{i=1}^k
\frac{\rho_i}{\wt_i\wvar}
\left(\|B_i (\wt_i\tilde{\mx}_i\tilde{\mx}_i^T)\| +
\|B_i A_i\| + \|D_i B_i\|\right) + O(\wvar^{-2}) \\
& \leq & k \|\rho - 1\bar{\rho}\|_\infty \rol + \sum_{i=1}^k
\frac{\rho_i}{\wt_i\wvar} (\rol + \ol \rol + \rol) + O(\wvar^{-2})\\
& \leq & k\|\rho - 1\bar{\rho}\|_\infty \rol + \frac{3k\bar{\rho}}{\wmin\wvar} \rol\\
& \leq & \frac{7k}{2\wmin\wvar} \rol + O(\wvar^{-2}).
\end{eqnarray*}
Thus, we have $\max\{\|\Delta_{11}\|,\|\Delta_{22}\|,\|\Delta_{21}\|\}
\leq 4k\rol/(\wmin\wvar) + O(\alpha^{-2})$, so that
\[
\|M - \Gamma\| \leq \|\Delta\| + O(\alpha^{-2}) \leq 2
\max\{\|\Delta_{11}\|,\|\Delta_{22}\|,\|\Delta_{21}\|\} \leq
\frac{8k}{\wmin\wvar} \rol + O(\wvar^{-2}) \leq
\frac{16k}{\wmin\wvar} \rol.
\]
for sufficiently large $n$, as $\alpha \geq n/\wmin$.
\end{proof}

\subsection{Sample Convergence}\label{sec:sample-convergence}
We now give some bounds on the convergence of the transformation to
isotropy ($\hat{\mu} \rightarrow 0$ and $\hat{\Sx} \rightarrow I$) and
on the convergence of the reweighted sample mean $\hat{u}$ and sample
matrix of second moments $\hat{M}$ to their expectations $u$ and $M$.
For the convergence of second moment matrices, we use the following
lemma due to Rudelson \cite{Rudelson1999}, which was presented in this
form in \cite{RuVe2007}.

\begin{lemma}\label{rud-covariance-tail}
Let $y$ be a random vector from a distribution $D$ in $\R^n$, with
$\sup_D \|y\| = M$ and $\|\E(yy^T) \| \le 1$. Let $y_1, \ldots, y_m$ be
independent samples from $D$. Let
\[
\eta = CM\sqrt{\frac{\log m}{m}}
\]
where $C$ is an absolute constant. Then,
\begin{enumerate}
\item[(i)] If $\eta < 1$, then
\[
\E\left(\|\frac{1}{m}\sum_{i=1}^my_i y_i^T - \E(yy^T)\|\right) \le \eta.
\]
\item[(ii)] For every $t \in (0, 1)$,
\[
\Pr\left(\|\frac{1}{m}\sum_{i=1}^my_i y_i^T - \E(yy^T)\| > t\right) \le 2e^{-ct^2/\eta^2}.
\]
\end{enumerate}
\end{lemma}

This lemma is used to show that a distribution can be made nearly
isotropic using only $O^*(kn)$ samples \cite{Rudelson1999,
Lovasz2007}.  The isotropic transformation is computed simply by
estimating the mean and covariance matrix of a sample, and computing
the affine transformation that puts the sample in isotropic position.

\begin{theorem}\label{thrm:isotropy}
There is an absolute constant $C$ such that for an isotropic mixture of $k$ logconcave distributions, with probability at least $1-\delta$, a sample of size
\[
m > C\frac{kn\log^2(n/\delta)}{\eps^2}
\]
gives a sample mean $\hat{\mu}$ and sample covariance $\hat{\Sx}$ so that
\[
\|\hat{\mu}\| \le \eps \quad \mbox{ and } \quad \|\hat{\Sx}-I\| \le \eps.
\]
\end{theorem}

We now consider the reweighted moments.

\begin{lemma}\label{lemma:mean-convergence}
Let $\epsilon, \delta > 0$ and let $\hat{\mu}$ be the reweighted
sample mean of a set of $m$ points drawn from an isotropic mixture of
$k$ Gaussians in $n$ dimensions, where
\[
m \geq \frac{2n\alpha}{\epsilon^2} \log \frac{2n}{\delta}.
\]
Then
\[
\Pr\left[
\left\| \hat{u} - u \right\| > \epsilon \right]
\leq \delta
\]
\end{lemma}
\begin{proof}
We first consider only a single coordinate of the vector $\hat{u}$.
Let $y = x_1 \exp\left(-\|x\|^2/(2\wvar)\right) - u_1$.  We observe that
\[
\left|x_1 \exp\left(-\frac{\|x\|^2}{2\wvar}\right)\right| \leq
|x_1| \exp\left(-\frac{x_1^2}{2\wvar}\right) \leq
\sqrt{\frac{\alpha}{e}} < \sqrt{\alpha}.
\]
Thus, each term in the sum $m \hat{u}_1 = \sum_{j=1}^m y_j$ falls the
range $[-\sqrt{\alpha} - u_1 ,\sqrt{\alpha} - u_1]$.  We may
therefore apply Hoeffding's inequality to show that
\[
\Pr\left[ |\hat{u}_1 - u_1| \geq \epsilon/ \sqrt{n} \right] \leq
2 \exp\left( -\frac{ 2 m^2 (\epsilon/\sqrt{n})^2}{m \cdot (2 \sqrt{\alpha})^2}\right)
\leq 2 \exp\left( -\frac{ m \epsilon^2 }{2\alpha n }\right) \leq \frac{\delta}{n}.
\]
Taking the union bound over the $n$ coordinates, we have that with
probability $1-\delta$ the error in each coordinate is at most
$\epsilon/\sqrt{n}$, which implies that $\| \hat{u} - u\| \leq
\epsilon$.
\end{proof}

\begin{lemma}\label{lemma:covar-convergence}
Let $\epsilon, \delta > 0$ and let $\hat{M}$ be the reweighted sample
matrix of second moments for a set of $m$ points drawn from an
isotropic mixture of $k$ Gaussians in $n$ dimensions, where
\[
m \geq C_1\frac{n\alpha}{\epsilon^2} \log \frac{n\alpha}{\delta}.
\]
and $C_1$ is an absolute constant.
Then
\[
\Pr\left[ \left\| \hat{M} - M\right\| > \epsilon \right] < \delta.
\]
\end{lemma}
\begin{proof}
We will apply Lemma \ref{rud-covariance-tail}.
Define $y = x \exp\left( -\|x\|^2/(2\wvar) \right)$.  Then,
\[
y_i^2\le x_i^2 \exp\left(-\frac{\|x\|^2}{\wvar}\right) \leq
x_i^2 \exp\left(-\frac{x_i^2}{\wvar}\right) \leq
\frac{\alpha}{e} < \alpha.
\]
Therefore $\|y\| \le \sqrt{\alpha n}$.

Next, since $M$ is in isotropic position (we can assume this w.l.o.g.), we have for any unit vector $v$,
\[
\E((v^Ty)^2)) \le \E((v^Tx)^2) \le 1
\]
and so $\|\E(yy^T)\| \le 1$.

Now we apply the second part of Lemma \ref{rud-covariance-tail} with
$\eta = \eps\sqrt{c/\ln(2/\delta)}$ and $t =
\eta\sqrt{\ln(2/\delta)/c}$. This requires that
\[
\eta = \frac{c\eps}{\ln(2/\delta)} \le C\sqrt{\alpha n}\sqrt{\frac{\log m}{m}}
\]
which is satisfied for our choice of $m$.
\end{proof}

\begin{lemma}\label{lemma:dir-convergence}
Let $X$ be a collection of $m$ points drawn from a Gaussian with mean
$\mx$ and variance $\sigma^2$.  With probability $1-\delta$,
\[
|x - \mx| \leq \sigma  \sqrt{2\log m/\delta}.
\]
for every $x \in X$.
\end{lemma}

\subsection{Perturbation Lemma} \label{sec:perturb-lemmas}
We will use the following key lemma due to Stewart \cite{Stewart1990}
to show that when we apply the spectral step, the top $k-1$
dimensional invariant subspace will be close to the Fisher subspace.

\begin{lemma}[Stewart's Theorem]\label{lemma:stewart}
Suppose $A$ and $A + E$ are n-by-n symmetric matrices and that
\[
\begin{array}{ccc}
A = &
\left[
\begin{array}{cc}
D_1 & 0\\
0 & D_2
\end{array}
\right]
&
\begin{array}{c}
r \\
n-r
\end{array}\\
&
\begin{array}{cc}
r & n-r
\end{array}
& \\
\end{array}
\;\;
\begin{array}{ccc}
E = &
\left[
\begin{array}{cc}
E_{11} & E_{21}^T\\
E_{21} & E_{22}
\end{array}
\right]
&
\begin{array}{c}
r \\
n-r
\end{array}\\
&
\begin{array}{cc}
r & n-r
\end{array}
& \\
\end{array}.
\]
Let the columns of $V$ be the top $r$ eigenvectors of the matrix $A +
E$ and let $P_2$ be the matrix with columns $e_{r+1},\ldots,e_n$.
If $d = \lambda_r(D_1) - \lambda_1(D_2) > 0$ and
\[
\|E\| \leq \frac{d}{5},
\]
then
\[
\|V^T P_2\| \leq \frac{4}{d} \|E_{21}\|_2.
\]
\end{lemma}

\section{Finding a Vector near the Fisher Subspace}\label{sec:find-dir}
In this section, we combine the approximations of Section
\ref{sec:M-approx} and the perturbation lemma of Section
\ref{sec:perturb-lemmas} to show that the direction $h$ chosen by step
3 of the algorithm is close to the intermean subspace.  Section
\ref{sec:recursion} argues that this direction can be used to
partition the components.  Finding the separating direction is the
most challenging part of the classification task and represents the
main contribution of this work.

We first assume zero overlap and that the sample reweighted moments
behave exactly according to expectation.  In this case, the mean shift
$\hat{u}$ becomes
\[
v \equiv \sum_{i=1}^k \wt_i \rho_i \mx_i.
\]
We can intuitively think of the components that have greater $\rho_i$
as gaining mixing weight and those with smaller $\rho_i$ as losing
mixing weight.  As long as the $\rho_i$ are not all equal, we will
observe some shift of the mean in the intermean subspace, i.e. Fisher
subspace.  Therefore, we may use this direction to partition the
components.  On the other hand, if all of the $\rho_i$ are equal, then
$\hat{M}$ becomes
\[
\Gamma \equiv
\left[
\begin{array}{cc}
\sum_{i=1}^k \rho_i ( \wt_i \tilde{\mx_i}\tilde{\mx_i}^T + A_i) & 0 \\
0 & \sum_{i=1}^k \rho_i D_i - \frac{\rho_i}{\wt_i\wvar} D_i^2\\
\end{array}
\right] =
\bar{\rho} \left[
\begin{array}{cc}
I & 0 \\
0 &  I - \frac{1}{\wvar} \sum_{i=1}^k  \frac{1}{\wt_i} D_i^2\\
\end{array}
\right].
\]
Notice that the second moments in the subspace
$\Span\{e_1,\ldots,e_{k-1}\}$ are maintained while those in the
complementary subspace are reduced by $\Poly(1/\alpha)$.  Therefore,
the top eigenvector will be in the intermean subspace, which is the
Fisher subspace.

We now argue that this same strategy can be adapted to work in
general, i.e., with nonzero overlap and sampling errors, with high
probability.  A critical aspect of this argument is that the norm of
the error term $\hat{M} - \Gamma$ depends only on $\ol$ and $k$ and
not the dimension of the data.  See Lemma \ref{lemma:M2-approx} and
the supporting Lemma \ref{lemma:covar-structure} and Fact
\ref{fact:off-diag-bound}.

Since we cannot know directly how imbalanced the $\rho_i$ are, we
choose the method of finding a separating direction according the norm
of the vector $\|\hat{u}\|$.  Recall that when $\|\hat{u}\| >
\sqrt{\wmin}/(32 \wvar)$ the algorithm uses $\hat{u}$ to determine the
separating direction $h$.  Lemma \ref{lemma:mean-shift} guarantees
that this vector is close to the Fisher subspace.  When $\|\hat{u}\|
\leq \sqrt{\wmin}/(32 \wvar)$, the algorithm uses the top eigenvector
of the covariance matrix $\hat{M}$.  Lemma \ref{lemma:spectral}
guarantees that this vector is close to the Fisher subspace.

\begin{lemma}[Mean Shift Method]\label{lemma:mean-shift}
Let $\epsilon > 0$.  There exists a constant $C$ such that if $m_1
\geq C n^4 \Poly(k,\wmin^{-1},\log n/\delta)$, then the following
holds with probability $1 - \delta$.  If $\|\hat{u}\| >
\sqrt{\wmin}/(32 \wvar)$ and
\[
\ol \leq \frac{\wmin^2 \epsilon}{2^{14} k^2},
\]
then
\[
\frac{\|\hat{u}^T v\|}{\|\hat{u}\|\|v\|} \geq 1 -  \epsilon.
\]
\end{lemma}

\begin{lemma}[Spectral Method]\label{lemma:spectral}
Let $\epsilon > 0$.  There exists a constant $C$ such that if $m_1
\geq C n^4 \Poly(k,\wmin^{-1},\log n/\delta)$, then the following
holds with probability $1 - \delta$.  Let $v_1,\ldots,v_{k-1}$ be the
top $k-1$ eigenvectors of $\hat{M}$.  If $\|\hat{u}\| \leq
\sqrt{\wmin}/(32 \wvar)$ and
\[
\ol \leq \frac{\wmin^2 \epsilon} {640^2 k^2}
\]
then
\[
\min_{v \in \Span\{v_1,\ldots,v_{k-1}\},\|v\| = 1}\|\proj_F (v)\| \geq
1 - \epsilon.
\]
\end{lemma}

\subsection{Mean Shift}
\begin{proof}[Proof of Lemma \ref{lemma:mean-shift}]
We will make use of the following claim.
\begin{claim}\label{claim:dot-prod-bound}
For any vectors $a,b \neq 0$,
\[
\frac{|a^T b|}{\|a\|\|b\|} \geq
\left(1 - \frac{\|a - b\|^2}{\max\{\|a\|^2,\|b\|^2\}}\right)^{1/2}.
\]
\end{claim}

By the triangle inequality, $\|\hat{u} - v\| \leq \|\hat{u} - u\| +
\|u - v\|$.  By Lemma \ref{lemma:M1-approx},
\[
\|u - v\|
\leq \sqrt{\frac{4 k^2}{\wvar^2\wmin} \ol}
 = \sqrt{\frac{4 k^2}{\wvar^2\wmin} \cdot \frac{\wmin^2 \epsilon}{2^14k^2}}
\leq \sqrt{\frac{\wmin \epsilon}{2^{12}\wvar^2}}.
\]
By Lemma \ref{lemma:mean-convergence}, for large $m_1$ we obtain the
same bound on $\|\hat{u} - u\|$ with probability $1-\delta$ .  Thus,
\[
\|\hat{u} - v\| \leq \sqrt{\frac{\wmin \epsilon}{2^{10}\wvar^2}}.
\]
Applying the claim gives
\begin{eqnarray*}
\frac{\|\hat{u}^T v\|}{\|\hat{u}\|\|v\|}
& \geq & 1 - \frac{\|\hat{u} - v\|^2}{\|\hat{u}\|^2}\\
& \geq & 1 - \frac{\wmin \epsilon}{2^{10} \wvar^2 } \cdot \frac{32^2 \wvar^2}{\wmin}\\
& = & 1 - \epsilon.
\end{eqnarray*}
\end{proof}

\begin{proof}[Proof of Claim \ref{claim:dot-prod-bound}]
Without loss of generality, assume $\|u\| \geq \|v\|$ and fix the
distance $\|u - v\|$.  In order to maximize the angle between $u$ and
$v$, the vector $v$ should be chosen so that it is tangent to the
sphere centered at $u$ with radius $\|u - v\|$.  Hence, the vectors
$u$,$v$,$(u-v)$ form a right triangle where $\|u\|^2 = \|v\|^2 + \|u -
v\|^2$.  For this choice of $v$, let $\theta$ be the angle between $u$
and $v$ so that
\[
\frac{u^T v}{\|u\|\|v\|} = \cos \theta = (1 - \sin^2 \theta)^{1/2} =
\left(1 - \frac{\|u - v\|^2}{\|u\|^2}\right)^{1/2}.
\]
\end{proof}

\subsection{Spectral Method}

We first show that the smallness of the mean shift $\hat{u}$ implies
that the coefficients $\rho_i$ are sufficiently uniform to allow us to
apply the spectral method.

\begin{claim}[Small Mean Shift Implies Balanced Second Moments]
\label{claim:smsibsm}
If $\|\hat{u}| \leq \sqrt{\wmin}/(32 \wvar)$ and
\[
\rol \leq \frac{\wmin}{64 k},
\]
then
\[
\| \rho - 1\bar{\rho}\|_2 \leq \frac{1}{8\wvar}.
\]
\end{claim}
\begin{proof}
Let $q_1,\ldots,q_k$ be the right singular vectors of the matrix $U =
\left[\wt_1 \mx_1,\ldots,\wt_k \mx_k\right]$ and let $\sigma_i(U)$ be
the $i$th largest singular value.  Because $\sum_{i=1}^k \wt_i \mx_i =
0$, we have that $\sigma_k(U) = 0$ and $q_k = 1/\sqrt{k}$.  Recall
that $\rho$ is the $k$ vector of scalars $\rho_1,\ldots,\rho_k$ and
that $v = U\rho$.  Then
\begin{eqnarray*}
\|v\|^2 & = & \|U \rho\|^2 \\
& = & \sum_{i=1}^{k-1} \sigma_i(U)^2 (q_i^T \rho)^2\\
& \geq & \sigma_{k-1}(U)^2 \| \rho - q_k(q_k^T \rho)\|_2^2 \\
& = & \sigma_{k-1}(U)^2 \| \rho - 1\bar{\rho}\|_2^2.
\end{eqnarray*}
Because $q_{k-1} \in \Span\{ \mx_1,\ldots,\mx_{k}\}$, we have that
$\sum_{i=1}^k \wt_i q_{k-1}^T \mx_i \mx_i^T q_{k-1} \geq 1 - \ol$.
Therefore,
\begin{eqnarray*}
\sigma_{k-1}(U)^2 & = & \|U q_{k-1}\|^2\\
& = & q_{k-1}^T \left(\sum_{i=1}^k \wt_i^2 \mx_i \mx_i^T\right) q_{k-1}\\
& \geq & \wmin
q_{k-1}^T  \left(\sum_{i=1}^k \wt_i \mx_i \mx_i^T\right) q_{k-1}\\
& \geq & \wmin (1 - \ol).
\end{eqnarray*}
Thus, we have the bound
\[
\|\rho - 1\bar{\rho}\|_\infty \leq
\frac{1}{\sqrt{(1-\ol)\wmin}} \|v\|\leq
\frac{2}{\sqrt{\wmin}} \|v\|.
\]
By the triangle inequality $\|v\| \leq \|\hat{u}\| + \|\hat{u} - v\|$.
As argued in Lemma \ref{lemma:M1-approx},
\[
\|\hat{u} - v\|
\leq \sqrt{\frac{4 k^2}{\wvar^2\wmin}\ol} = 
\sqrt{\frac{4 k^2}{\wvar^2\wmin} \cdot \frac{\wmin^2}{64^2k^2}} = 
\leq \frac{\sqrt{\wmin}}{32 \wvar}.
\]
Thus,
\begin{eqnarray*}
\| \rho - 1\bar{\rho}\|_\infty & \leq &
\frac{2\bar{\rho}}{\sqrt{\wmin}} \|v\|\\
& \leq &
\frac{2\bar{\rho}}{\sqrt{\wmin}}
\left(\frac{\sqrt{\wmin}}{32 \wvar} + \frac{\sqrt{\wmin}}{32\wvar}\right)\\
& \leq & \frac{1}{8\wvar}.
\end{eqnarray*}
\end{proof}

We next show that the top $k-1$ principal components of $\Gamma$ span
the intermean subspace and put a lower bound on the spectral gap
between the intermean and non-intermean components.

\begin{lemma}[Ideal Case]\label{lemma:Gamma-is-good}
If $\|\rho -1\bar{\rho}\|_\infty \leq 1/(8\wvar)$, then
\[
\lambda_{k-1}(\Gamma) - \lambda_{k}(\Gamma) \geq \frac{1}{4\wvar},
\]
and the top $k-1$ eigenvectors of $\Gamma$ span the means of the components.
\end{lemma}

\begin{proof}[Proof of Lemma \ref{lemma:Gamma-is-good}]
We first bound $\lambda_{k-1}(\Gamma_{11})$.  Recall that
\[
\Gamma_{11} = \sum_{i=1}^k \rho_i ( \wt_i \tilde{\mx_i}\tilde{\mx_i}^T + A_i).
\]
Thus,
\begin{eqnarray*}
\lambda_{k-1}(\Gamma_{11}) & = &
\min_{\|y\| = 1}
\sum_{i=1}^k \rho_i y^T (\wt_i\tilde{\mu_i}\tilde{\mu_i}^T + A_i )y\\
& \geq & \bar{\rho} - \max_{\|y\| = 1} \sum_{i=1}^k (\bar{\rho} - \rho_i) y^T
(\wt_i \tilde{\mu_i}\tilde{\mu_i}^T + A_i)y.
\end{eqnarray*}
We observe that $\sum_{i=1}^k y^T (\wt_i\tilde{\mu_i}\tilde{\mu_i}^T +
A_i )y = 1$ and each term is non-negative.  Hence the sum is bounded
by
\[
\sum_{i=1}^k (\bar{\rho} - \rho_i) y^T (\wt_i
\tilde{\mu_i}\tilde{\mu_i}^T + A_i)y \leq \|\rho -
1\bar{\rho}\|_\infty,
\]
so,
\[
\lambda_{k-1}(\Gamma_{11})
\geq \bar{\rho} - \|\rho - 1\bar{\rho}\|_\infty.
\]

Next, we bound $\lambda_1(\Gamma_{22})$.  Recall that
\[
\Gamma_{22} =\sum_{i=1}^k \rho_i D_i - \frac{\rho_i}{\wt_i\wvar} D_i^2
\]
and that for any $n-k$ vector $y$ such that $\|y\| = 1$, we have
$\sum_{i=1}^k y^TD_iy = 1$.  Using the same arguments as above,
\begin{eqnarray*}
\lambda_1(\Gamma_{22}) & = & \max_{\|y\| = 1}
\bar{\rho} + \sum_{i=1}^k (\rho_i - \bar{\rho}) y^TD_iy
- \frac{\rho_i}{\wt_i\wvar} y^TD_i^2y\\
& \leq & \bar{\rho} + \|\rho - 1\bar{\rho}\|_\infty - \min_{\|y\| = 1}
\sum_{i=1}^k \frac{\rho_i}{\wt_i\wvar} y^TD_i^2y.
\end{eqnarray*}
To bound the last sum, we observe that $\rho_i - \bar{\rho} =
O(\wvar^{-1})$.  Therefore
\[
\sum_{i=1}^k \frac{\rho_i}{\wt_i\wvar} y^TD_i^2y \geq
\frac{\bar{\rho}}{\wvar} \sum_{i=1}^k \frac{1}{\wt_i} y^T D_i^2 y + O(\wvar^{-2}).
\]
Without loss of generality, we may assume that $y = e_1$ by an
appropriate rotation of the $D_i$.  Let $D_i(\ell,j)$ be element in
the $\ell$th row and $j$th column of the matrix $D_i$.  Then the sum
becomes
\begin{eqnarray*}
\sum_{i=1}^k \frac{1}{\wt_i} y^T D_i^2 y & = &
\sum_{i=1}^k \frac{1}{\wt_i} \sum_{j=1}^n D_j(1,j)^2\\
& \geq & \sum_{i=1}^k \frac{1}{\wt_i} D_j(1,1)^2.
\end{eqnarray*}
Because $\sum_{i=1}^k D_i = I$, we have $\sum_{i=1}^k D_i(1,1) = 1$.  From
the Cauchy-Schwartz inequality, it follows
\[
\left(\sum_{i=1}^k \wt_i \right)^{1/2} \left(\sum_{i=1}^k \frac{1}{\wt_i} D_i(1,1)^2\right)^{1/2} \geq \sum_{i=1}^k \sqrt{\wt_i} \frac{D_i(1,1)}{\sqrt{\wt_i}} = 1.
\]
Since $\sum_{i=1}^k \wt_i =1$, we conclude that $\sum_{i=1}^k
\frac{1}{\wt_i} D_i(1,1)^2 \geq 1$.  Thus, using the fact that
$\bar{\rho} \geq 1/2$, we have
\[
\sum_{i=1}^k \frac{\rho_i}{\wt_i\wvar} y^TD_i^2y \geq \frac{1}{2 \wvar}
\]

Putting the bounds together
\[
\lambda_{k-1}(\Gamma_{11}) - \lambda_1(\Gamma_{22})
\geq \frac{1}{2\wvar} - 2\|\rho - 1\bar{\rho}\|_\infty
\geq \frac{1}{4\wvar}.
\]
\end{proof}

\begin{proof}[Proof of Lemma \ref{lemma:spectral}]
To bound the effect of overlap and sample errors on the eigenvectors,
we apply Stewart's Lemma (Lemma \ref{lemma:stewart}).  Define
$d = \lambda_{k-1}(\Gamma) - \lambda_k (\Gamma)$ and $E = \hat{M} -
\Gamma$.

We assume that the mean shift satisfies $\|\hat{u}\| \leq
\sqrt{\wmin}/(32 \alpha)$ and that $\ol$ is small.  By Lemma
\ref{lemma:Gamma-is-good}, this implies that
\begin{equation}
d = \lambda_{k-1}(\Gamma) - \lambda_k (\Gamma) \geq \frac{1}{4\alpha}.
\label{eqn:d}
\end{equation}

To bound $\|E\|$, we use the triangle inequality $\|E\| \leq \|\Gamma
- M\| + \|M - \hat{M}\|$.  Lemma \ref{lemma:M2-approx} bounds the first
term by
\[
\|M - \Gamma\| \leq \sqrt{\frac{16^2 k^2}{\wmin^2\wvar^2} \ol} 
= \sqrt{\frac{16^2 k^2}{\wmin^2\wvar^2} \cdot \frac{\wmin^2
\epsilon}{640^2 k^2}} 
\leq \frac{1}{40\wvar} \sqrt{\epsilon}.
\label{eqn:E}
\]
By Lemma \ref{lemma:covar-convergence}, we obtain the same bound on
$\|M - \hat{M}\|$ with probability $1-\delta$ for large enough $m_1$.
Thus,
\[
\|E\| \leq  \frac{1}{20\wvar} \sqrt{\epsilon}.
\]

Combining the bounds of Eqn. \ref{eqn:d} and \ref{eqn:E}, we have
\[
\sqrt{1 - (1-\epsilon)^2}d - 5 \|E\| \geq
\sqrt{1 - (1-\epsilon)^2} \frac{1}{4\wvar}
- 5 \frac{1}{20\wvar} \sqrt{\epsilon} \geq 0,
\]
as $\sqrt{1 - (1-\epsilon)^2} \geq \sqrt{\epsilon}$.  This implies
both that $\|E\| \leq d/5$ and that $4 \|E_{21}|/d < \sqrt{1 -
(1-\epsilon)^2}$, enabling us to apply Stewart's Lemma to the matrix
pair $\Gamma$ and $\hat{M}$.

By Lemma \ref{lemma:Gamma-is-good}, the top $k-1$ eigenvectors of
$\Gamma$, i.e. $e_1,\ldots,e_{k-1}$, span the means of the components.
Let the columns of $P_1$ be these eigenvectors.  Let the columns of
$P_2$ be defined such that $[P_1, P_2 ]$ is an orthonormal matrix and
let $v_1,\ldots,v_k$ be the top $k-1$ eigenvectors of $\hat{M}$.  By
Stewart's Lemma, letting the columns of $V$ be $v_1,\ldots,v_{k-1}$, we have
\[
\|V^T P_2\|_2 \leq \sqrt{1 - (1-\epsilon)^2},
\]
or equivalently,
\[
\min_{v \in \Span\{v_1,\ldots,v_{k-1}\},\|v\| = 1}\|\proj_F v\| =
\sigma_{k-1}(V^T P_1) \geq 1 - \epsilon.
\]
\end{proof}

\section{Recursion}\label{sec:recursion}
In this section, we show that for every direction $h$ that is close to
the intermean subspace, the ``largest gap clustering'' step produces a
pair of complementary halfspaces that partitions $\mathbb{R}^n$ while
leaving only a small part of the probability mass on the wrong side of
the partition, small enough that with high probability, it does not
affect the samples used by the algorithm.

\begin{lemma}\label{lemma:gap-clustering}
Let $\delta,\delta' > 0$, where $\delta' \leq \delta/(2m_2)$, and let
$m_2$ satisfy $m_2 \geq n/k \log (2k/\delta)$.  Suppose that $h$ is a
unit vector such that
\[
\|\proj_F (h)\| \geq 1 - \frac{\wmin}{2^{10}(k-1)^2
\log \frac{1}{\delta'}}.
\]
Let $\Mix$ be a mixture of $k > 1$ Gaussians with overlap
\[
\ol \leq \frac{\wmin}{2^9(k-1)^2} \log^{-1} \frac{1}{\delta'}.
\]
Let $X$ be a collection of $m_2$ points from $\Mix$ and let $t$ be the
midpoint of the largest gap in set $\{h^T x : x \in X\}$.  With
probability $1-\delta$, the halfspace $H_{h,t}$ has the following
property.  For a random sample $y$ from $\Mix$ either
\[
y,\mx_{\ell(y)} \in H_{h,t} \mbox{ or } y,\mx_{\ell(y)} \notin H_{h,t}
\]
with probability $1 - \delta'$.
\end{lemma}

\begin{proof}[Proof of Lemma \ref{lemma:gap-clustering}]
The idea behind the proof is simple.  We first show that two of the
means are at least a constant distance apart.  We then bound the width
of a component along the direction $h$, i.e. the maximum distance
between two points belonging to the same component.  If the width of
each component is small, then clearly the largest gap must fall
between components.  Setting $t$ to be the midpoint of the gap, we
avoid cutting any components.

We first show that at least one mean must be far from the origin in
the direction $h$.  Let the columns of $P_1$ be the vectors
$e_1,\ldots, e_{k-1}$.  The span of these vectors is also the span of
the means, so we have
\begin{eqnarray*}
\max_i (h^T \mx_i)^2 & = & \max_i (h^T P_1P_1^T\mx_i)^2\\
& = & \|P_1^T h\|^2 \max_i
\left( \frac{(P_1^T h)^T}{\|P_1 h\|} \tilde{\mx}_i\right)^2\\
& \geq & \|P_1^T h\|^2 \sum_{i=1}^k \wt_i \left(\frac{(P_1^T h)^T}{\|P_1 h\|} \tilde{\mx}_i\right)^2\\
& \geq & \|P_1^T h\|^2 (1-\ol)\\
& > & \frac{1}{2}.
\end{eqnarray*}
Since the origin is the mean of the means, we conclude that the
maximum distance between two means in the direction $h$ is at least
$1/2$.  Without loss of generality, we assume that the interval
$[0,1/2]$ is contained between two means projected to $h$.

We now show that every point $x$ drawn from component $i$ falls in a
narrow interval when projected to $h$.  That is, $x$ satisfies $h^T x \in
b_i$, where $b_i = [h^T \mx_i - (8(k-1))^{-1}, h^T \mx_i +
(8(k-1))^{-1}]$.  We begin by examining the variance along $h$.  Let
$e_k,\ldots,e_n$ be the columns of the matrix $n$-by-$(n-k+1)$ matrix
$P_2$.  Recall from Eqn. \ref{eqn:covar-structure} that $P_1^T \wt_i
\Sx_i P_1 = A_i$, that $P_2^T \wt_i \Sx_i P_1 = B_i$, and that $P_2^T
\wt_i \Sx_i P_2 = D_i$.  The norms of these matrices are bounded
according to Lemma \ref{lemma:covar-structure}. Also, the vector $h =
P_1P_1^T h + P_2P_2^T h$.  For convenience of notation we define
$\epsilon$ such that $\|P_1^T h\| = 1 -\epsilon$.  Then $ \|P_2^T
h\|^2 = 1 - (1-\epsilon)^2 \leq 2\epsilon$.  We now argue
\begin{eqnarray*}
h^T \wt_i \Sx_i h & \leq &
\left(h^T P_1 A_i P_1^T h + 2 h^T P_2 B_i P_1 h + h^T P_2^T D_i P_2 h\right)\\
& \leq & 2 \left(h^T P_1 A_i  P_1^T h + h^T P_2 D_i P_2^T h \right)\\
& \leq & 2 (\|P_1^T h\|^2\|A_i\| + \|P_2^T h\|^2\| \|D_i\|)\\
& \leq & 2 (\ol + 2\epsilon).
\end{eqnarray*}
Using the assumptions about $\ol$ and $\epsilon$, we conclude that the
maximum variance along $h$ is at most
\[
\max_i h^T \Sx_i h \leq
\frac{2}{\wmin} \left(
\frac{\wmin}{2^9(k-1)^2} \log \frac{1}{\delta'}
+ 2 \frac{\wmin}{2^{10}(k-1)^2} \log \frac{1}{\delta'}\right) \leq \left(2^7(k-1)^2 \log 1/\delta' \right)^{-1} .
\]

We now translate these bounds on the variance to a bound on the
difference between the minimum and maximum points along the direction
$h$.  By Lemma \ref{lemma:dir-convergence}, with probability $1-\delta/2$
\[
|h^T (x - \mx_{\ell(x)})| \leq 
\sqrt{2 h^T \Sx_i h \log( 2 m_2/\delta)}
\leq \frac{1}{8(k-1)} \cdot \frac{\log(2 m_2/\delta)}{\log(1/\delta')} 
\leq \frac{1}{8(k-1)}.
\]
Thus, with probability $1 - \delta/2$, every point from $X$ falls into the
union of intervals $b_1 \cup \ldots \cup b_k$ where $b_i = [h^T \mx_i
- (8(k-1))^{-1}, h^T \mx_i + (8(k-1))^{-1}]$.  Because these intervals
are centered about the means, at least the equivalent of one interval
must fall outside the range $[0,1/2]$, which we assumed was contained
between two projected means.  Thus, the measure of subset of $[0,1/2]$
that does not fall into one of the intervals is
\[
\frac{1}{2} - (k-1)\frac{1}{4(k-1)} = \frac{1}{4}.
\]
This set can be cut into at most $k-1$ intervals, so the smallest
possible gap between these intervals is $(4(k-1))^{-1}$, which is
exactly the width of an interval.

Because $m_2 = k/\wmin \log (2k/\delta)$ the set $X$ contains at least
one sample from every component with probability $1-\delta/2$.
Overall, with probability $1-\delta$ every component has at least one
sample and all samples from component $i$ fall in $b_i$.  Thus, the
largest gap between the sampled points will not contain one of the
intervals $ b_1,\ldots,b_k$.  Moreover, the midpoint $t$ of this gap
must also fall outside of $b_1 \cup \ldots \cup b_k$, ensuring that no
$b_i$ is cut by $t$.

By the same argument given above, any single point $y$ from $\Mix$ is
contained in $b_1 \cup \ldots \cup b_k$ with probability $1-\delta'$
proving the Lemma.
\end{proof}

In the proof of the main theorem for large $k$, we will need to have
every point sampled from $\Mix$ in the recursion subtree
classified correctly by the halfspace, so we will assume $\delta'$
considerably smaller than $m_2/\delta$.

The second lemma shows that all submixtures have smaller overlap to
ensure that all the relevant lemmas apply in the recursive steps.

\begin{lemma}\label{lemma:overlap-is-monotonic}
The removal of any subset of components cannot induce a mixture with
greater overlap than the original.
\end{lemma}

\begin{proof}[Proof of Lemma \ref{lemma:overlap-is-monotonic}]
Suppose that the components $j+1, \ldots k$ are removed from the
mixture.  Let $\omega = \sum_{i=1}^j \wt_i$ be a normalizing factor
for the weights. Then if $c = \sum_{i=1}^j \wt_i \mx_i = -
\sum_{i=j+1}^k \wt_i \mx_i$, the induced mean is $\omega^{-1}c$.  Let
$T$ be the subspace that minimizes the maximum overlap for the full
$k$ component mixture.  We then argue that the overlap
$\tilde{\ol}^2$ of the induced mixture is bounded by
\begin{eqnarray*}
\tilde{\ol} & = & \min_{\dim(S) = j-1} \max_{v \in S}
\frac{\omega^{-1}v^T \Sx v}
{\omega^{-1}\sum_{i=1}^j \wt_i v^T (\mx_i\mx_i^T - cc^T +\Sx_i) v}\\
 & \leq &
\max_{v \in  \Span\{e_1,\ldots,e_{k-1}\}\setminus\Span\{\mu_{j+1},\ldots,\mu_k\}}
\frac{\sum_{i=1}^j \wt_i v^T \Sx_i v}
{\sum_{i=1}^j \wt_i v^T (\mx_i\mx_i^T -cc^T+\Sx_i) v}.
\end{eqnarray*}
Every $v \in
\Span\{e_1,\ldots,e_{k-1}\}\setminus\Span\{\mu_{j+1},\ldots,\mu_k\}$
must be orthogonal to every $\mx_\ell$ for $j+1 \leq \ell \leq k$.
Therefore, $v$ must be orthogonal to $c$ as well.  This also enables
us to add the terms for $j+1,\ldots,k$ in both the numerator and
denominator, because they are all zero.
\begin{eqnarray*}
\tilde{\ol}  & \leq &
\max_{v \in \Span\{e_1,\ldots,e_{k-1}\}\setminus\Span\{\mu_{j+1},\ldots,\mu_k\}}
\frac{v^T \Sx v }
{\sum_{i=1}^k \wt_i v^T (\mx_i\mx_i^T +\Sx_i) v}\\
& \leq &
\max_{v \in \Span\{e_1,\ldots,e_{k-1}\}}
\frac{v^T \Sx v}
{\sum_{i=1}^k \wt_i v^T (\mx_i\mx_i^T +\Sx_i) v}\\
& = & \ol.
\end{eqnarray*}
\end{proof}

The proofs of the main theorems are now apparent.  Consider the case
of $k=2$ Gaussians first.  As argued in Section
\ref{sec:sample-convergence}, using $m_1 = \omega(k n^4 \wmin^{-3}
\log (n/\delta\wmin))$ samples to estimate $\hat{u}$ and $\hat{M}$ is
sufficient to guarantee that the estimates are accurate.  For a
well-chosen constant $C$, the condition
\[
\ol \leq J(p) \leq C \wmin^3 \log^{-1} \left(\frac{1}{\delta\wmin} + \frac{1}{\eta}\right)
\]
of Theorem \ref{thrm:k=2-fisher} implies that
\[
\rol \leq \frac{\wmin\sqrt{\epsilon}}{640 \cdot 2},
\]
where
\[
\epsilon = \frac{\wmin}{2^9}
\log^{-1} \left(\frac{2 m_2}{\delta} + \frac{1}{\eta}\right).
\]
The arguments of Section \ref{sec:find-dir} then show that the
direction $h$ selected in step 3 satisfies
\[
\|P_1^T h\| \geq 1 - \epsilon = 1 - \frac{\wmin}{2^9}
\log^{-1} \left(\frac{m_2}{\delta} + \frac{1}{\eta}\right).
\]
Already, for the overlap we have
\[
\rol \leq \frac{\wmin\sqrt{\epsilon}}{640 \cdot 2} \leq
\sqrt{\frac{\wmin}{2^9(k-1)^2}} \log^{-1/2} \frac{1}{\delta'}.
\]
so we may apply Lemma \ref{lemma:gap-clustering} with $\delta' =
(m_2/\delta + 1/\eta)^{-1}$.  Thus, with probability $1-\delta$ the
classifier $H_{h,t}$ is correct with probability $1 - \delta' \geq 1
-\eta$.

We follow the same outline for $k > 2$, with the quantity $1/\delta' =
m_2/\delta + 1/\eta$ being replaced with $1/\delta' = m/\delta +
1/\eta$, where $m$ is the total number of samples used.  This is
necessary because the half-space $H_{h,t}$ must classify every sample
point taken below it in the recursion subtree correctly.  This adds
the $n$ and $k$ factors so that the required overlap becomes
\[
\ol \leq C \wmin^3 k^{-3}\log^{-1} \left(\frac{nk}{\delta\wmin} +
\frac{1}{\eta}\right)
\]
for an appropriate constant $C$.  The correctness in the recursive
steps is guaranteed by Lemma \ref{lemma:overlap-is-monotonic}.
Assuming that all previous steps are correct, the termination
condition of step 4 is clearly correct when a single component is
isolated.

\section{Conclusion}
We have presented an affine-invariant extension of principal components. We expect that this technique should be applicable to a broader class of problems. For example, mixtures of distributions with some mild properties such as center symmetry and some bounds on the first few moments might be solvable using isotropic PCA. It would be nice to characterize the full scope of the technique for clustering and also to find other applications, given that standard PCA is widely used.

\bibliographystyle{plain}
\bibliography{gaussians}

\begin{thebibliography}{10}

\bibitem{Achlioptas2005}
D.~Achlioptas and F.~McSherry.
\newblock On spectral learning of mixtures of distributions.
\newblock In {\em Proc. of COLT}, 2005.

\bibitem{Chaudhuri2008}
K.~Chaudhuri and S.~Rao.
\newblock Learning mixtures of product distributions using correlations and
  independence.
\newblock In {\em Proc. of COLT}, 2008.

\bibitem{Dasgupta2005}
Anirban Dasgupta, John Hopcroft, Jon Kleinberg, and Mark Sandler.
\newblock On learning mixtures of heavy-tailed distributions.
\newblock In {\em Proc. of FOCS}, 2005.

\bibitem{Dasgupta1999}
S.~DasGupta.
\newblock Learning mixtures of gaussians.
\newblock In {\em Proc. of FOCS}, 1999.

\bibitem{Dempster1977}
A.P. Dempster, N.M. Laird, and D.B. Rubin.
\newblock Maximum likelihood from incomplete data via the em algorithm.
\newblock {\em Journal of the Royal Statistical Society B}, 39:1--38, 1977.

\bibitem{Feldman2006}
Jon Feldman, Rocco~A. Servedio, and Ryan O'Donnell.
\newblock Pac learning axis-aligned mixtures of gaussians with no separation
  assumption.
\newblock In {\em COLT}, pages 20--34, 2006.

\bibitem{Fukunaga1990}
K.~Fukunaga.
\newblock {\em Introduction to Statistical Pattern Recognition}.
\newblock Academic Press, 1990.

\bibitem{Duda2001}
R.~O. Duda~P.E. Hart and D.G. Stork.
\newblock {\em Pattern Classification}.
\newblock John Wiley \& Sons, 2001.

\bibitem{Kannan2005}
R.~Kannan, H.~Salmasian, and S.~Vempala.
\newblock The spectral method for general mixture models.
\newblock In {\em Proceedings of the 18th Conference on Learning Theory}.
  University of California Press, 2005.

\bibitem{Lovasz2007}
L.~Lov\'{a}sz and S.~Vempala.
\newblock The geometry of logconcave functions and and sampling algorithms.
\newblock {\em Random Strucures and Algorithms}, 30(3):307--358, 2007.

\bibitem{MacQueen1967}
J.~B. MacQueen.
\newblock Some methods for classification and analysis of multivariate
  observations.
\newblock In {\em Proceedings of 5-th Berkeley Symposium on Mathematical
  Statistics and Probability}, volume~1, pages 281--297. University of
  California Press, 1967.

\bibitem{Rudelson1999}
M.~Rudelson.
\newblock Random vectors in the isotropic position.
\newblock {\em Journal of Functional Analysis}, 164:60--72, 1999.

\bibitem{RuVe2007}
M.~Rudelson and R.~Vershynin.
\newblock Sampling from large matrices: An approach through geometric
  functional analysis.
\newblock {\em J. ACM}, 54(4), 2007.

\bibitem{Arora2005}
R.~Kannan S.~Arora.
\newblock Learning mixtures of arbitrary gaussians.
\newblock {\em Ann. Appl. Probab.}, 15(1A):69--92, 2005.

\bibitem{Dasgupta2000}
L.~Schulman S.~DasGupta.
\newblock A two-round variant of em for gaussian mixtures.
\newblock In {\em Sixteenth Conference on Uncertainty in Artificial
  Intelligence}, 2000.

\bibitem{Stewart1990}
G.W. Stewart and Ji~guang Sun.
\newblock {\em Matrix Perturbation Theory}.
\newblock Academic Press, Inc., 1990.

\bibitem{Vempala2002}
S.~Vempala and G.~Wang.
\newblock A spectral algorithm for learning mixtures of distributions.
\newblock {\em Proc. of FOCS 2002; JCCS}, 68(4):841--860, 2004.

\end{thebibliography}

\appendix

\end{document}